\documentclass[12pt]{article} 

\usepackage[utf8]{inputenc}
\usepackage{fullpage}
\usepackage{natbib}
\usepackage{xcolor,xspace}
\usepackage{dsfont}
\usepackage[normalem]{ulem}

\usepackage{algpseudocode}
\usepackage{algcompatible}
\usepackage{algorithm}

\usepackage{amsmath,amssymb,mathtools}
\usepackage{amsmath,amsthm,amssymb,mathtools,thmtools}
\usepackage{nicefrac}

\newcommand{\partialh}[1]{\text{p}(#1)}
\newcommand{\U}{\mathcal{U}}
\newcommand{\V}{\mathcal{V}}
\newcommand{\W}{\mathcal{W}}
\newtheorem{theorem}{Theorem}
\newtheorem{lemma}[theorem]{Lemma}

\newtheorem{remark}[theorem]{Remark}

\newtheorem{definition}[theorem]{Definition}

   \newcommand{\reals}{\mathbb{R}}
   \newcommand{\naturals}{\mathbb{N}}
   \newcommand{\Ex}{\mathbb{E}}
   \renewcommand{\Pr}{\mathbb{P}}
   \newcommand{\Lo}[1]{{\mathcal L_{#1}}}
   \newcommand{\bLo}[1]{{\mathcal L^{0/1}_{#1}}}
   \newcommand{\rLo}[2]{{\mathcal L^{#1}_{#2}}}
   
   \newcommand{\blo}{\ell^{0/1}}
   \newcommand{\rlo}[1]{\ell^{#1}}

    \newcommand{\indct}[1]{\mathds{1}\left[{#1}\right]}
   \newcommand{\B}{{\mathcal B}}
   \renewcommand{\P}{{\mathcal P}}
   \newcommand{\A}{{\mathcal A}}

  \renewcommand{\H}{{\mathcal H}}
  \newcommand{\C}{{\mathcal C}}

  \newcommand{\vc}{\mathrm{VC}}

  \renewcommand{\d}{\mathrm{dist}}

  \newcommand{\iid}{i.i.d.~}
   
  \newcommand{\nn}{\mathrm{nn}}
  \newcommand{\sm}{\mathrm{sm}}
   \newcommand{\argmin}{\mathrm{argmin}}

\newcommand{\cH}{\mathcal{H}}

\newcommand{\bR}{\mathbb{R}}

\title{Simplifying Adversarially Robust PAC Learning with Tolerance}
\usepackage{times}

\author{
Hassan Ashtiani \thanks{McMaster University, \texttt{zokaeiam@mcmaster.ca}. Hassan Ashtiani is also a faculty affiliate at Toronto's Vector Institute and was supported by an NSERC Discovery Grant.} \and
Vinayak Pathak \thanks{Independent, \texttt{path.vinayak@gmail.com}.} \and
Ruth Urner \thanks{York University, \texttt{ruth@eecs.yorku.ca}. Ruth Urner is also a faculty affiliate at Toronto's Vector Institute and was supported by an NSERC Discovery Grant.}
}
\begin{document}

\maketitle

\begin{abstract}%

 {Adversarially robust PAC learning has proved to be challenging, with the currently best known learners~\citep{montasser2021adversarially} relying on improper methods based on intricate compression schemes, resulting in sample complexity exponential in the VC-dimension.}
 A series of follow up work considered a slightly relaxed version of the problem called adversarially robust learning \emph{with tolerance}~\citep{ashtiani2023adversarially, bhattacharjee2023robust, raman2024proper} and achieved better sample complexity in terms of the VC-dimension. However, those algorithms were either improper and complex, or required additional assumptions on the hypothesis class $\H$. We prove, for the first time, the existence of a simpler learner that achieves a sample complexity linear in the VC-dimension without requiring additional assumptions on $\H$. Even though our learner is improper, it is ``almost proper" in the sense that it outputs a hypothesis that is ``similar" to a hypothesis in $\H$. 

We also use the ideas from our algorithm to construct a semi-supervised learner in the tolerant setting. This simple algorithm achieves comparable bounds to the previous (non-tolerant) semi-supervised algorithm of~\cite{attias2022characterization}, 
but avoids the use of intricate subroutines from previous works, and is ``almost proper."

\end{abstract}

\section{Introduction}
In standard PAC-learning, the user encodes their domain knowledge by specifying a hypothesis class $\H$ that they think achieves a small expected loss on the data distribution. In adversarially robust PAC-learning, in addition to $\H$, the user also has knowledge of some perturbation type $\U: X\rightarrow 2^X$ encoding a belief that all points in the close environment $\U(x)$ of point $x$ share the same label as $x$.

Imagine, for example, that $X$ is a domain of images, and $\U(x)$ is a small $\ell_0$-ball around $x$. This encodes the domain knowledge that one should not be able to change the label of an image by changing a small number of pixels. Thus, for standard PAC-learning, the goal is to find a hypothesis $h$ that achieves a small expected binary loss $\Ex_{(x,y)\sim P}\indct{h(x)\neq y}$, whereas for adversarially robust PAC-learning the goal is to get a small expected adversarial loss $\Ex_{(x,y)\sim P}\indct{\exists z\in\U(x): h(z)\neq y}$.

Unlike standard PAC learning, adversarially robust PAC-learning of VC classes requires an improper learner in general, and the best known learner~\citep{MontasserHS19} uses a potentially exponential number of samples in the VC-dimension of the class. To address this, \cite{ashtiani2023adversarially} defined a mild relaxation of the problem where the expected adversarial loss (with respect to $\U$) of the learner is compared with the best achievable adversarial loss with respect to a slightly larger perturbation type $\V$.  
It was argued that the user is typically impartial to aim for robustness with respect to $\U$ versus $\V$ if they are very ``similar''. 
A \emph{tolerance} parameter $\gamma$ was then introduced to capture the relationship between $\U$ and $\V$ (see Section~\ref{sec:tolerant} for a precise definition). It has been shown~\citep{ashtiani2023adversarially, bhattacharjee2023robust} that tolerant adversarial learning can be achieved with a number of samples that is linear in the VC-dimension, $\log(\frac{1}{\gamma})$ and ambient dimension $d$. However, such a bound has so far only been achieved through either a complex compression-based algorithm~\citep{ashtiani2023adversarially}, or by adding additional assumptions on $\H$ and $\U$~\citep{bhattacharjee2023robust} (such as the property of ``regularity" that has appeared in several other works \citep{awasthi2021existence,raman2024proper} under varying names).

In our work, we show for the first time a \emph{simple} learning algorithm that achieves sample complexity linear in VC, $\log(\frac{1}{\gamma})$, and $d$ and does not require any assumption on $\H$. In the realizable case, the algorithm is easy to state: run Robust ERM (RERM) on the training set to get a hypothesis $h\in\H$ and output a ``smoothed'' version of $h$: for a (perturbed) test point $z$ output the majority over a small set $\W(z)$ around $z$. For the agnostic case, we present a slight modification using a discrete ``cover" of the space. 
Even though our learners are improper, the improperness only appears in the last smoothing step, meaning the final output is simply a smoothed version of a hypothesis in $\H$.

Next, we extend our ideas to the semi-supervised setting. In an earlier work~\citep{attias2022characterization} it was proved that semi-supervised adversarially robust learning has a small labeled sample complexity as long as a large number of unlabeled samples are present. However, their algorithm involved invoking complex subroutines such as the one-inclusion-graph algorithm. We show that if we add tolerance then a simpler semi-supervised learner achieves comparable bounds.

\subsection{Related work}
Adversarially robust PAC-learning was formulated to study an empirical phenomenon, first encountered in image classification: state of the art models were vulnerable to adversarial attacks, namely imperceptible perturbations of an input image that led the otherwise highly accurate model to erroneously change its output~\citep{SzegedyZSBEGF13}. Adversarial robustness has since developed into a fertile area of research in the last decade. 
However, developing sound and efficient practical methods as well as obtaining a theoretical understanding of the problem remain challenging.

From a theoretical perspective, the sample complexity of adversarial PAC learning~\citep{feige2015learning, MontasserHS19} has been widely investigated~\citep{feige2015learning, attias2018improved, ashtiani2020black,montasser2021adversarially}. However, the best known upper bounds~\citep{MontasserHS19, montasser2022adversarially} are based on rather involved and impractical learning methods, namely on intricate compression schemes~\citep{moran2016sample} or one-inclusion-graphs.
Also, the known sample complexity bounds are exponential in VC-dimension of the hypothesis class. 
Variations of the problem such as semi-supervised learning~\citep{ashtiani2020black, attias2022characterization} and learning real-valued functions~\citep{attias2023adversarially} have also been studied and similar upper bounds have been derived. It can be easily shown that VC-dimension does not provide a lower bound, since any class with infinite VC-dimension is trivially learnable when
$\U(x)$ is the entire domain $X$. A dimension characterizing adversarially robust PAC-learning was obtained in~\citep{montasser2022adversarially}, but the dimension is based on a global variant of the one-inclusion graph~\citep{haussler1994predicting} 
with potentially infinite vertices and edges. A simpler characterization remains elusive.

A major drawback of the standard PAC framing of adversarial robustness is the fixation on one perturbation type, which realistically cannot be known by the learner. Various alternatives have been investigated, such as robustness that is adaptive to the underlying distribution~\citep{BhattacharjeeC21adaptive} or robustness with respect to a large collection of perturbation sets~\citep{montasser2021adversarially,lechner2024adversarially}. For the latter approach, it was shown that under structural assumptions on the class or perturbation types (such as a linear ordering by inclusion) and access to additional oracles (a perfect attack oracle provides witness points to adversarial vulnerability) adversarially robust learning is still possible.

Adversarial learning with tolerance was introduced in~\cite{ashtiani2023adversarially} and further studied by~\cite{bhattacharjee2023robust, raman2024proper}. A related notion (corresponding to the special case of tolerance parameter $\gamma = 1$ in our terminology) was considered in additional studies~\citep{montasser2021transductive,blum2022boosting}. Another relaxation of the adversarial problem was studied in~\cite{robey2022probabilistically, raman2024proper} where the requirement is robustness to the majority of perturbations in a perturbation set (as opposed to all perturbations in the set).

\section{Notations and Setup}
\label{s:notations}

We denote by $X$ the input domain (often $X=\bR^d$) and by $Y=\{0,1\}$ the binary label space. We assume that $X$ is equipped with a metric $\d$. A hypothesis $h:X\to Y$ is a function that assigns a label to each point in the domain.
A hypothesis class $\H$ is a set of hypotheses. For a sample $S = ((x_1, y_1), \ldots, (x_n, y_n))\in (X\times Y)^n$, we use the notation $S_X =(x_1, x_2, \ldots, x_n)$ to denote the collection of domain points $x_i$ in $S$. The binary (also called 0-1) loss of $h$ on data point $(x,y)\in X\times Y$ is defined by
$\blo(h, x, y) = \indct{h(x) \neq y}$,
where $\indct{.}$ is the indicator function. Let $P$ by a probability distribution over $X\times Y$. Then the \emph{expected binary loss} of $h$ with respect to $P$ is defined by
$\bLo{P} (h) = \Ex_{(x,y)\sim P} [\blo(h , x, y)]$.
Similarly, the \emph{empirical binary loss} of $h$ on sample $S = ((x_1, y_1), \ldots, (x_n, y_n)$ is defined as $\bLo{S}(h) = \frac{1}{n}\sum_{i=1}^n \blo(h, x_i, y_i)$. We let the \emph{approximation error} of $\cH$ with respect to $P$ be denoted by $\bLo{P} (\cH) = \inf_{h\in \cH}\bLo{P} (h)$. 

A \emph{learner} $\A$ is a function that takes in a finite sequence of labeled instances $S \in(X\times Y)^n$ and outputs a hypothesis $h = \A(S)$. See Appendix Section \ref{appsec:vctheory} to recall the standard PAC learning requirement for binary classification~\citep{vapnikcherv71, Valiant84}.

\subsection{Tolerant Adversarial PAC Learning}
\label{sec:tolerant}

In robust learning under adversarial perturbations (or simply adversarial learning), it is assumed that an adversary can replace a test point $x$ with any point $z$ in $\U(x)$, where $\U(x)\subseteq X $ is a predefined set of ``admissible perturbations" for $x$.
We call the function $\U:X\to 2^X$ the \emph{perturbation type}. 
\begin{definition}[Adversarial loss ]
The adversarial loss of $h$ with respect to $\U$ on $(x,y)\in X\times Y$ is defined by $\rlo{\U}(h, x, y) = \max_{z\in\U(x)} \blo(h, z, y)$. The \emph{expected adversarial loss} with respect to $P$ is defined by $\rLo{\U}{P}(h)=\Ex_{(x,y)\sim P}\rlo{\U}(h, x, y)$. Similarly, the \emph{empirical adversarial loss} of $h$ on sample $S = ((x_1, y_1), \ldots, (x_n, y_n))\in (X\times Y)^n$ is defined by $\rLo{\U}{S}(h) = \frac{1}{n}\sum_{i=1}^n \rlo{\U}(h, x_i, y_i)$. Finally, the \emph{adversarial approximation error} of $\cH$ with respect to $\U$ and $P$ is defined by $\rLo{\U}{P} (\cH) = \inf_{h\in \cH}\rLo{\U}{P} (h)$.
\end{definition}
The adversarial loss encompasses a classifier mispredicting on an instance and the instance being too close to the classifier's decision boundary. The following definition isolates the latter component.
\begin{definition}[Margin loss]
    Given $h\in\H, x\in X$, and a perturbation type $\U$, we define margin loss $\ell^{\U, \mathrm{mar}}(h, x) = \indct{\exists x_1, x_2\in\V(x): h(x_1)\neq h(x_2)}$. We define $\rLo{\U,\mathrm{mar}}{S}$ and $\rLo{\U, \mathrm{mar}}{P}$ accordingly.
\end{definition}
Analogously to PAC learning for the binary loss (Definition \ref{def:learn}), one can define PAC learning with respect to the adversarial loss. We here define the more general setting of \emph{tolerant} adversarial learning. Consider two perturbation types $\U$ and $\V$. We say $\U$ is \emph{contained in} $\V$ and write it as $\U \prec \V$ if $\U(x)\subseteq\V(x)$ for all $x\in X$.
Introducing tolerance relaxes adversarial learning by comparing the robust loss of the algorithm with respect to $\U$ with the approximation error of $\H$ with respect to a larger perturbation type $\V$.

\begin{definition}[Tolerant Adversarial PAC Learner\citep{ashtiani2023adversarially}]\label{def:adv_learn_tol}
Let $\P$ be a set of distributions over $X\times Y$, $\cH$ a hypothesis class, and $\U \prec \V$ two perturbation types.
We say $\A$ $(\U, \V)$-\emph{tolerantly} PAC learns $\H$ with respect to $\P$ with $m_\A: (0,1)^2\to \mathbb{N}$ samples if the following holds:
for every distribution $P\in\P$ and every $\epsilon,\delta \in (0,1)$, if $S$ is an \iid sample of size at least $m_\A(\epsilon, \delta)$ from $P$, then with probability at least $1-\delta$ (over the randomness of $S$) we have
\[
\rLo{\U}{P}(\A(S)) \leq \rLo{\V}{P}(\cH) + \epsilon.
\]
We say $\A$ is a tolerant PAC learner in the \emph{agnostic setting} if $\P$ is the set of all distributions over $X\times Y$, and in the \emph{tolerantly realizable setting} if $\P=\{P:\rLo{\V}{P}(\cH) = 0\}$.
\end{definition}
We call $\U$ the \emph{actual perturbation type} and $\V$ the \emph{reference perturbation type}. The above definition recovers the standard definition of PAC learning under adversarial perturbations \citep{MontasserHS19} when $\U(x)=\V(x)$ for all $x\in X$. 
The smallest function $m: (0,1)^2\to \mathbb{N}$ for which there exists a learner $\A$ that satisfies the above definition with $m_{\A} = m$ is referred to as the (realizable or agnostic) \emph{sample complexity} of the problem.

In our work, we often consider a specific form of actual and reference perturbation types $\U$ and $\V$, namely $\V$ resulting from $\U$ by ``inflating" it with a third perturbation type $\W$. Let $\W$ and $\U$ be two perturbation types. 
We now define $\V$ by setting $\V(x) = \{x'' ~|~ \exists x'\in\U(x)~\text{st}~x''\in\W(x')\}$\footnote{The result of \cite{montasser2021transductive} for transductive adversarial learning can be thought of as a result in the tolerant setting for $\W(x)=\U^{-1}(x)=\{x'~|~x\in\U(x')\}$. }.
Inspired by its role in our methods, we also refer to $\W$ as the \emph{smoothing perturbation type}.

Perturbation types are defined naturally when $X$ is equipped with a metric $\d(.,.)$. In this case, $\U(x)$ can be defined by a ball of radius $r$ around $x$, i.e.,  $\U(x)=\B_r(x) = \{z\in X ~\mid~ \d(x,z) \leq r\}$. Now one can inflate $\U(x)$ with $\W(x) = \B_{r\gamma}(x)$ to create $\V(x) = \B_{(1+\gamma)r}(x)$. The perturbation types that we consider in this paper are mostly of this form. We call $\gamma>0$ the \emph{tolerance parameter} and we will refer to $(\U,\V)$-tolerance also as \emph{$\gamma$-tolerance} in this case. 

We will further assume that our actual, reference and smoothing perturbation types $\U, \V$ and $\W$ are so that the perturbation sets $\U(x), \V(x)$ and $\W(x)$ admit the definition of a uniform measure over them. We use the notation $\mu_{\U(x)}, \mu_{\V(x)}$ and $\mu_{\W(x)}$ for these measures. 
We will use the notation $x'\sim \U(x)$ etc to denote sampling form these uniform measures of the perturbation sets.

The following complexity measure was  introduced by~\cite{montasser2021adversarially} and adopted in various works~\citep{attias2022characterization, shao2022theory} with different names. We adopt the one used by~\cite{attias2022characterization}.
The definition immediately yields $\vc_\U(\H) \leq \vc(\H)$ for all $\H$ and $\U$. 
\begin{definition}[$\vc_\U$-dimension]\label{def:vc_u}
Let $\U:X\to 2^X$ be some perturbation type and $\H\subseteq \{0,1\}^X$ a hypothesis class. We say that $\H$ $\U$-shatters a set of points $K\subseteq X$ if for every labeling $y\in \{0,1\}^K$ there exists a function $h\in\H$ with $h(z) = y(x)$ for all $z\in\U(x)$ and all $x\in K$. The $\vc_\U$-dimension of the class $\H$ is the supremum over the sizes of sets that $\H$ can $\U$-shatter. 
\end{definition}
Additional discussion of tolerance, and its relation to smoothed analysis, is in Appendix Section \ref{appsec:why-tolerance}.

\subsection{Empirical Risk Minimization (ERM) and basic VC theory}
It is well known that, for binary classification a class $\H$ is PAC learnable if and only if its \emph{VC-dimension} is finite and that such classes can be learned through \emph{Empirical Risk Minimization (ERM)}. More generally, ERM is a successful PAC learning principle whenever the VC-dimension of the \emph{loss class $\H_\ell$} of a class $\H$ induced by loss function $\ell:\{0,1\}^X \times X \times Y \to \{0,1\}$ is finite. 
This induced loss class is a collection of subsets of $X\times Y$ defined by
$\H_\ell = \{h_\ell \subseteq X\times Y ~:~ h \in \H\}$, where $\quad h_\ell = \{(x,y)\in X\times Y ~:~ \ell(h,x,y) =1)\}$. 
For the adversarial loss with respect to a perturbation type $\V$ we will use the notation $\H_\V$ to denote the loss class of $\H$ with respect to $\rlo{\V}$.

Our learners often employ empirical risk minimzation with respect to various losses as a subroutine. We thus define the following notation for ERM (and Robust ERM) oracles.

\begin{definition}[ERM and RERM oracles]
    Let $S =  ((x_1, y_1), \ldots, (x_m, y_m))$.
    An ERM oracle $\A_\H$ with respect to the hypothesis class $\H$ outputs any $\hat{h} \in \arg\min_{h\in\H} \bLo{S}(h)$.
    An RERM (Robust ERM) oracle $\A_\H^\V$ with respect to class $\H$ and perturbation type $\V$ outputs any $\hat{h}\in\arg\min_{h\in\H} \rLo{\V}{S}(h)$.
\end{definition}

A standard result in VC theory is that a collection of subsets of finite VC-dimension enjoys finite sample uniform convergence~\citep{vapnikcherv71}. This immediately implies that any learner $\A$ that is an empirical risk minimizer for a class $\H$ with respect to loss $\ell$ (that is, $\A$ always outputs an $h\in\H$ with minimal empirical loss) is a successful PAC learner for $\H$ with respect to $\ell$.
See appendix Section \ref{appsec:vctheory} for a reminder of this classic PAC learning result and the definition of the VC dimension.
While the VC-dimension of the loss class $\H_\U$ for adversarial robust losses $\rlo{\U}$ can be arbitrarily larger than $\vc(\H)$ \citep{MontasserHS19}, it has been shown that the VC-dimension of the loss class for \emph{finite} perturbation types can be bounded, a key compnent in our analysis.
\begin{lemma}[\citep{attias2018improved} Lemma 1]\label{prop:vc_finite_loss_class}
    Let $\H\subseteq\{0,1\}^X$ be some hypothesis class and let $\C:X\to 2^X$ be a perturbation type that satisfies $|\C(x)|\leq k$ for all $x\in X$ for some $k\in\naturals$. Then the VC-dimension of the robust loss class is bounded by $\vc(\H_C) \leq \vc(\H)\log(k)$.
\end{lemma}

\section{Supervised Tolerant Learning}

The algorithms we propose for tolerantly robust learning of VC classes are \emph{improper} learners. A common aspect of our methods is a two-stage approach, where the learner first determines an RERM hypothesis based on the training data, and then performs a post-processing step on this RERM hypothesis to obtain the final predictor, which in turn is not from the class $\H$. We show that this aspect of non-properness is \emph{necessary} for any successful tolerantly robust learner, even when the perturbation types are balls in a Euclidean space and the sample complexity can additionally depend on the dimension of the space. The proof of the impossibility below result is in Appendix Section \ref{appsec_lower_bound_proof}

\begin{theorem}\label{thm:lower_bound}
For any $r\in\reals$, any $d\in \naturals$ and any $g>0$, there exist a hypothesis class $\H$ over $X = \reals^d$ with $\vc(\H) = 1$ that is not properly tolerantly robustly PAC learnable (even in the tolerantly realizable case) for $\U(x) = \B_r(x)$ and $\V(x) = \B_{(1+\gamma)r}(x)$ for any $\gamma$ with $0<\gamma \leq g$.
\end{theorem}

\subsection{Supervised tolerant learning in the realizable case}
We start by presenting a simple robust learner, outlined in Algorithm \ref{alg:supervised} below, for the tolerantly realizable setting. 
It proceeds in two stages. Given data $S$, it first determines an RERM hypothesis $\hat{h}$ with respect to the reference  perturbation type $\V$. 
It then smoothes this hypothesis by assigning each domain point $x$ the majority label of $\hat{h}$ in $\W(x)$ for a smoothing perturbation type $\W$. 
\begin{algorithm}
\caption{RERM-and-Smooth}\label{alg:supervised}
\begin{algorithmic}
\STATE {\bf Input:} Data $S =  ((x_1, y_1), \ldots, (x_m, y_m))$, access to an RERM oracle $\A^\V_\H$, and smoothing perturbation type $\W$.
\STATE Set $\hat{h} = \A^\V_{\H}(S)$
\STATE {\bf Output:} $\sm_{\W}(\hat{h})$ defined by 
\STATE \qquad\qquad $\sm_{\W}(\hat{h})(x) = \indct{\Ex_{x'\sim\W(x)} \hat{h}(x') \geq 1/2}$
\end{algorithmic}
\end{algorithm}

The analysis of the method in Algorithm \ref{alg:supervised} employs the notion of $\eta$-nets for the pre-images of the label classes under hypothesis class $\H$. For a hypothesis $h\in\{0,1\}^X$, we let $h_0 = h^{-1}(0) =  \{x\in X ~:~ h(x) = 0\}$ denote the pre-image of label $0$ under $h$, and analogously we let $h_1 = h^{-1}(1) = \{x\in X ~:~ h(x) = 1\}$ denote the pre-image of label $1$ under $h$. Using this notation, we define the following version of an $\eta$-net for a binary hypothesis class. 

\begin{definition}[$\eta$-net for class $\H$]\label{def:eta-net}
Let $\H\subseteq\{0,1\}^X$ be some hypothesis class, let $D$ be a distribution over $X$ and let $\eta>0$. A domain subset $C\subseteq X$ is an \emph{$\eta$-net} for $\H$ with respect to $D$ if whenever $D(h_y) \geq \eta$ for some $y\in\{0,1\}$ and $h\in \H$, then $h_y\cap C\neq \emptyset$.
\end{definition}
An $\eta$-net $C$ for $\H$ ensures that whenever a label-pre-image $h_0 = h^{-1}(0)$ or $h_1= h^{-1}(1)$ has mass at least $\eta$ under distribution $D$, then the net $C$ contains at least one point in this pre-image. Standard VC-theory guarantees that a sample of size $O(\vc(\H)/\eta)$ is a $\eta$-net for $\H$ with probability lower bounded by a constant, say $2/3$. This in particular means that for any distribution $D$ and class $\H$ there exists an $\eta$-net of size $O(\vc(\H)/\eta)$.

The analysis of Algorithm \ref{alg:supervised} above will employ a discrete perturbation type $\C\prec \V$, that has finite perturbation sets $\C(x)$, is included in the reference perturbation type $\V$, and is so that each set $\C(x)$ is an $\eta$-net for the uniform distribution $\mu_{\V(x)}$ over the perturbation set $\V(x)$.
By the above argument, the discrete type can be chosen so that the sizes of the perturbation sets are uniformly bounded, $|\C(x)| = O(\vc(\H)/\eta)$ for all $x\in X$.

\begin{theorem}\label{thm:supervised_realizable}
Let $\H$ be a hypothesis class of finite VC-dimension ($\vc(\H) <\infty$), let $0< \eta < 1/3$, and let $\V,\U$ and $\W$ be perturbation types that satisfy $\U\prec \V$,   
$\W(z)\in\V(x)$ for all $x\in X$ and $z\in \U(x)$, and $\mu_{\V(x)}(\W(z))\geq 3\eta$ for all $x\in X$ and $z\in\U(x)$.
Then Algorithm \ref{alg:supervised} $(\U,\V)$-tolerant robustly PAC learns $\H$ in the tolerantly realizable case with sample complexity bounded by 
    \[
    m(\epsilon, \delta) = 
    \tilde{O}\left(\frac{\vc(\H)\log(\vc(\H)/\eta) + \log(1/\delta)}{\epsilon} \right).
    \]
\end{theorem}

\begin{proof}
Let $\U, \V, \W$ and $\eta$ satisfy the conditions stated in the theorem.
We define a new perturbation type $\C\prec \V$ as follows: for $x\in X$ set $\C(x)$ to be an $\eta$-net for $\H$ with respect to $\mu_{\V(x)}$, the uniform distribution over $\V(x)$. As outlined above, since $\vc(\H)$ is finite, basic VC-theory guarantees that $\C$ exists and can be chosen so that $|\C(x)| = O(\vc(\H))/\eta$. We let $k\in\naturals$ denote a uniform upper bound on the sizes of the perturbation sets in $\C$, that is $|\C(x)|\leq k$ for all $x\in X$.
This implies that the VC-dimension of the loss class $\H_{\C}$ of $\H$ with respect to $\rlo{\C}$ is bounded by $\vc(\H)\log(k) = O(\vc(\H)\log(\vc(\H)/\eta))$ (Lemma \ref{prop:vc_finite_loss_class}).

Now, let $h$ be some hypothesis and let $\sm_{\W}(h)$ be the $\W$-smoothed version of $h$ as defined above. We next prove that for any distribution $P$, we have 
\begin{equation}\label{eqn:loss_relations_smooth}
\rLo{\U}{P}(\sm_{\W}(h))  ~\leq~ \rLo{\C}{P}(h) ~\leq~ \rLo{\V}{P}(h)
\end{equation}
The second inequality is immediate from $\C\prec \V$. To prove the the first inequality we will show that it holds pointwise, that is $\rlo{\U}(\sm_{\W}(h), x, y)  \leq \rlo{\C}(h,x,y)$ for all $(x, y)\in X\times Y$.

Indeed, assume that for some $(x,y) \in X \times Y$ 
we have $\rlo{\U}(\sm_{\W}(h), x, y) = 1$. This means that there exists some $z\in \U(x)$ with $\sm_{\W}(h)(z) \neq y$. This 
implies that $\indct{\Ex_{x'\sim\W(z)} {h}(x') \geq 1/2} \neq y$, which means that $\Pr_{x'\sim \W(x)}[{h}(x')\neq y] > 1/2$. Now 
$\mu_{\V(x)}(\W(z))\geq 3\eta$ 
together with $\W(z)\subseteq\V(x)$ 
(where $\mu_{\V(x)}$ is a uniform measure over $\V$) 
implies that $\Pr_{x'\sim \V(x)}[{h}(x')\neq y] \geq (3/2)\eta> \eta$. 
Now, since $\C(x)$ is an $\eta$-net with respect to $\mu_{\V(x)}$ for $\H$, this implies that there exists a $c\in \C(x)$ with ${h}(c) \neq y$. Thus, we have $ \rlo{\C}(h,x,y) = 1$ which is what we needed to show.

Note that, since $\C \prec \V$, any distribution $P$ that is realizable by $\H$ with respect to $\rlo{\V}$ (as assumed in the tolerantly realizable case) is also realizable with respect to $\rlo{\C}$.
Further, any RERM hypothesis with respect to $\rlo{\V}$ is also an RERM hypothesis class with respect to $\rlo{\C}$.
Now the above bound on the VC-dimension on the loss class with respect to $\rlo{\C}$ implies that any RERM learner is a successful robust PAC learner for $\C$, and thus, with the stated samples sizes, with high probability at least $1-\delta$ we have $\rLo{\C}{P} (\hat{h}) \leq \epsilon$. Now Equation \ref{eqn:loss_relations_smooth} implies $\rLo{\U}{P}(\sm_{\W}(\hat{h})) \leq \epsilon$ as required.
\end{proof}

\begin{remark}
  If $\V$, $\U$ and $\W$ are Euclidean balls of radii $r(1+\gamma)$, $r$ and $r\gamma$ respectively, we have $\mu_{\V(x)}(\W(z)) = \frac{\gamma^d}{(1+\gamma)^d}$. Choosing $\eta = \frac{1}{3(1+1/\gamma)^d}$ yields sample complexity bound
    $
    m(\epsilon, \delta) = 
    \tilde{O}\left(\frac{\vc(\H)(\log(\vc(\H)) + d \log(1+1/\gamma)) + \log(1/\delta)}{\epsilon} \right)
    $
with Algorithm \ref{alg:supervised} in the tolerantly realizable case.
\end{remark}

\begin{remark}\label{rem:attractively_agnostic}
We emphasize that the learner in Algorithm \ref{alg:supervised} is not employing the discrete perturbation type $\C$ and thus does not need knowledge (or ability to construct) $\C$. Rather $\C$ and its properties as providing local $\eta$-nets are solely a tool of the analysis.
\end{remark}

A modification of the method presented here can be shown to also provide tolerantly robust guarantees in the agnostic case (see appendix Section \ref{appsec:local_discretization} for details of this modification and analysis). However, the modified learner requires knowledge of the discrete perturbation type $\C$, thus we loose the attractive property from Remark \ref{rem:attractively_agnostic}. We next present an alternative method, which is slightly more complex,
for tolerantly robust learning for which we derive guarantees in both the realizable and agnostic case. For that, we employ a global discretization which, in many natural settings, can be chosen to consist of grid points, and thus the learner can readily  access this discretaztion. {This method also achieves a slightly better sample complexity by eliminating the $\log(\vc(\H))$ factor.}

\subsection{Supervised agnostic tolerant learning using a global discretization}\label{ss:global_discret}
In this section we present a tolerant learner that uses a global discretization of the space. More specifically, the method is based on a countable domain subset $C\subseteq X$ as a discretization of the space. We will show that if $C$ is a $r\gamma$-cover of $X = \reals^d$, that is, the discretization $C$ is such that for all $x\in X$ there exists a point $c\in C$ with $\d(x,c)\leq r\gamma$, our algorithm is a successful tolerant robust PAC learner for perturbation types $\U(x) = \B_r(x)$ and $\V(x) = \B_{(1+\gamma)r}(x)$.
Given a discretization $C\subseteq X$, we let $\C$ denote the induced discretization of perturbation type $\V$, that is $\C(x) =  \V(x)\cap C$.
{For concreteness, we may assume that $C$ consists of evenly spaced grid points of a grid with side-length $2 r \gamma/{d}$. This will be an $r\gamma$-cover with respect to any $\ell_p$-norm with $p\geq 1$, and the sizes of the sets $\C(x)$ will be uniformly bounded by $ k :=\C(x)\leq \left(\frac{(1+\gamma){d}}{\gamma}\right)^d = \Theta\left((1 + \frac 1 \gamma)^d d^{d} \right)$.}

Our proposed method, Algorithm \ref{alg:supervised_global_discret} below, acts in two stages. Given a training dataset $S$ and discretization $C$, it first determines an RERM hypothesis $\hat{h}$ with respect to the discretized perturbation type $\C$. Given $\hat{h}$, it then produces a discretized version of this predictor by assigning every domain point $x$ the $\hat{h}$ label of its nearest neighbor in $C$ (breaking ties arbitrarily). More precisely, for a set $C \subseteq X$ and hypothesis $h\in\{0,1\}^X$ we define the $C$-nearest neighbor discretized hypothesis $\nn_C(h)$ by
\[
\nn_C(h)(x) = h(z) \text{ for some } z \in \argmin_{c\in C}\d(x,c).
\]
 We note that, in general, $\nn_C(h)\notin \H$ even for hypotheses $h\in\H$. However the predictor $\nn_C(h)$ can be viewed as ``close to being from $\H$'' in the sense that  $\nn_C(h)$ is a discretized version of $h$, its decision boundary being moved slightly to pass along grid lines (or more generally voronoi cells) induced by $C$. In that sense, informally, our $\gamma$-tolerant learner is ``$\gamma$-close to being proper''.

\begin{algorithm}
\caption{RERM-and-Discretize}\label{alg:supervised_global_discret}
\begin{algorithmic}
\STATE {\bf Input:} Data $S =  ((x_1, y_1), \ldots, (x_m, y_m))$, discretization $C\subseteq X$, access to RERM oracle $\A^\C_\H$.
\STATE Set $\hat{h} = \A^\C_{\H}(S)$.
\STATE {\bf Output:} $h = \nn_{\C}(\hat{h})$ defined by
\STATE \qquad\qquad $\nn_{C}(\hat{h})(x) = \hat{h}(z)$ for some $z\in \argmin_{c\in C}\d(x,c)$
\end{algorithmic}
\end{algorithm}

\begin{theorem}
\label{thm:supervised-global}
    Let $\H$ be a hypothesis class of finite VC-dimension $\vc(\H) <\infty$ and let $\V$ and $\U$ be perturbation types that are balls of radii $r(1+\gamma)$ and $r$ respectively for some $r >0$. Let discretization $C$ be an $r\gamma$-cover of $X$, let $\C$ be the induced perturbation type, where $\C(x) = \V(x)\cap C$ for all $x\in X$ and let $k$ be a uniform upper bound on the perturbations sets in $\C$, that is $|\C(x)|\leq k$ for all $x\in X$. Then Algorithm \ref{alg:supervised_global_discret} $\gamma$-tolerant robustly PAC learns $\H$. Moreover the sample complexity in the tolerantly realizable case is bounded by
    $m(\epsilon, \delta) = \tilde{O}\left(\frac{\vc(\H)\log(k) + \log(1/\delta)}{\epsilon} \right)$
    and in the agnostic case by
$        n(\epsilon, \delta) = \tilde{O}\left(\frac{\vc(\H)\log(k) + \log(1/\delta)}{\epsilon^2} \right)
$.
\end{theorem}
\begin{proof}
Let $h$ be some hypothesis and let $\nn_C(h)$ be the $C$-discretized version of $h$ as defined above. We start by proving that for any distribution $P$, we have 
\begin{equation}\label{eqn:loss_relations}
\rLo{\U}{P}(\nn_C(h))  ~\leq~ \rLo{\C}{P}(h) ~\leq~ \rLo{\V}{P}(h)
\end{equation}
We show that the first inequality holds pointwise,  $\rlo{\U}(\nn_C(h), x, y)  \leq \rlo{\C}(h,x,y)$ for all $(x, y)\in X\times Y$ (the second follows from $\C\prec \V$).
Indeed, assume we have $\rlo{\U}(\nn_C(h), x,  y) = 1$ for some $(x, y)$. Then there exists some $z\in \U(x) $ with $\nn_C(h)(z) \neq y$, and thus for some $z' \in \argmin_{c\in C}\d(z, c)$ we have $h(z')\neq y$. Since $C$ is a $r\gamma$-cover of $X$ we know $\d(z, z') \leq r\gamma$. Further, since $z\in\U(x)$, we know $\d(x,z) \leq r$. Now the triangle inquality implies $\d(x, z') \leq \d(x,z) +\d(z,z') \leq r + r\gamma = r(1+\gamma)$. Thus $z'\in C\cap \V(x) = \C(x)$, and now $h(z')\neq y $ means $\rlo{\C}(h,x,y) = 1$. Thus we have $\rlo{\U}(\nn_C(h), x, y)  \leq \rlo{\C}(h,x,y)$ for all $(x, y)\in X\times Y$, which implies Equation \ref{eqn:loss_relations} above.

Since  $\C(x)\leq k$ for all $x\in X$, we have $\vc(\H_\C) \leq \vc(\H)\log(k)$
(Lemma \ref{prop:vc_finite_loss_class}).
For the realizable case, the result now follows exactly as in the proof of Theorem \ref{thm:supervised_realizable}.
For the agnostic case, note that Equation \ref{eqn:loss_relations} also implies that for any distribution $P$, 
$\rLo{\C}{P}(\H) ~\leq ~ \rLo{\V}{P}(\H)$.
Due to $\vc(\H_\C) \leq \vc(\H)\log(k)$, any empirical risk minimizing learner $\A_\H^\C$, PAC learns $\H$ with respect to $\rlo{C}$, yielding
$\rLo{P}{\C}(\hat{h}) \leq \rLo{\C}{P}(\H) +\epsilon ~\leq ~ \rLo{P}{\V}(\H) +\epsilon
$
with high probability at least $1-\delta$ over the training samples for the stated agnostic sample sizes.
Combining this expression with  Equation \ref{eqn:loss_relations}, we obtain 
$\rLo{\U}{P}(\nn_C(\hat{h})) ~\leq ~ \rLo{\C}{P}(\hat{h}) ~\leq~ \rLo{\V}{P}(\H) +\epsilon
$ as required.
\end{proof}

\begin{remark}
    As outlined above, in case of $X = \reals^d$ equipped with any $\ell_p$ norm, we can choose the $r\gamma$-cover $C$ so that $ k :=\C(x)\leq \left(\frac{(1+\gamma)\sqrt{d}}{2\gamma}\right)^d = \Theta\left((1 + \frac 1 \gamma)^d d^{d/2} \right)$, thus we obtain sample complexity bounds
    $ m(\epsilon, \delta) = \tilde{O}\left(\frac{\vc(\H)d(\log(1+\frac 1 \gamma)+ \log d) + \log(1/\delta)}{\epsilon} \right)$
in the tolerantly realizable case and
$    n(\epsilon, \delta) = \tilde{O}\left(\frac{\vc(\H)d(\log(1+\frac 1 \gamma)+ \log d) + \log(1/\delta)}{\epsilon^2} \right)
$
    in the agnostic case.
\end{remark}

\section{Semi-supervised learning}
In semi-supervised learning, in addition to the $m$ labeled examples $S_l = ((x_1, y_1),\ldots , (x_m, y_m))$, the learner $\A$ is also given $n$ unlabeled examples $S_u = (x_1, \ldots, x_n)$ drawn iid from the marginal $P_X$ of the data distribution. Analogous to the supervised setting, we can define agnostic and realizable, as well as their tolerant versions for the semi-supervised setting.  
The sample complexity of semi-supervised learning can be quantified by two functions: the number of labeled samples (defined by function $m_l(\epsilon, \delta)$) and the number of unlabeled samples (denoted by
$m_u(\epsilon, \delta)$).

The sample complexity of semi-supervised adversarially robust learning was studied in~\citep{attias2022characterization} where it was shown that in the realizable case, the labeled sample complexity with respect to a perturbation set $\U$ is characterized by  $\vc_\U$, see Definition \ref{def:vc_u}, which can be significantly smaller than the standard VC-dimension of a class. They also showed that in the agnostic case bounding the labeled sample complexity with $\vc_\U$ is impossible. On the other hand, they proposed a ``factor-$\alpha$ agnostic learner'' (a.k.a. a semi-agnostic learner) that guarantees $\rLo{\U}{P}(\A(S_l,S_u))\leq\alpha\cdot\rLo{\U}{P}(\H)+\epsilon$ for $\alpha=3$ and whose labeled sample complexity is characterized by $\vc_\U$.
Specifically, they showed that given a supervised agnostic learner with sample complexity $m(\epsilon, \delta)$, the realizable semi-supervised case can be solved with $m(\epsilon/3,\delta/2)$ unlabeled and $O\left(\frac{\vc_\U(\H)}{\epsilon}\log^2\frac{\vc_\U(\H)}{\epsilon}+\frac{\log 1/\delta}{\epsilon}\right)$ labeled samples, and the factor-3 agnostic semi-supervised case can be solved with $m(\epsilon/3,\delta/2)$ unlabeled and $O\left(\frac{\vc_\U(\H)}{\epsilon^2}\log^2\frac{\vc_\U(\H)}{\epsilon^2}+\frac{\log 1/\delta}{\epsilon^2}\right)$ labeled samples.

Their algorithm has two steps. First, they transform $\H$ into a ``partial" concept class~\citep{alon2022theory}, and use the 1-inclusion graph based learner from~\citep{alon2022theory} to learn a partial concept $h$ using the labeled set $S_l$. Next, they label $S_u$ using $h$ and invoke the compression-based adversarially robust agnostic learner from~\citep{montasser2021adversarially} on this set. While the algorithm is simple to describe, both of its steps invoke other algorithms that are complex and rely on improper learning. 
For a tolerantly robust learning guarantee, we can use in their second step the simpler supervised learner with a tolerant robustness guarantee, thus achieving an unlabeled sample complexity of $\tilde{O}\left(\frac{\vc(\H)d(\log(1+\frac 1 \gamma)+ \log d) + \log(1/\delta)}{\epsilon^2} \right)$.

However, we can show that with tolerance, their algorithms can be further simplified if we allow for an extra $\log\vc(\H)$ factor in the labeled sample complexity. For both realizable and agnostic settings, we now present algorithms that are easy to state, do not require using complex subroutines, and achieve a labeled sample complexity that depends logarithmically on $\vc(\H)$. Moreover, like the supervised case, our algorithms are ``almost proper."

\subsubsection{Realizable} 
Our main insight is that using just the unlabeled set $S_u$, we can identify a small, finite set $\H'$ of candidates from $\H$ such that robust learning $\H'$ suffices. To create $\H'$, we simply iterate through all robustly realizable labelings $y = (y_1,\ldots,y_m)$ of $S_u$ and call RERM on $((x_1,y_1),\ldots,(x_m,y_m))$.

As in the result of Theorem~\ref{thm:supervised-global} for Algorithm \ref{alg:supervised_global_discret}, our SSL method in Algorithm \ref{alg:semi-supervised-realizable} employs an $r\gamma$-cover $C$ as a global discretization of the space, and uses the induced perturbation type $\C$ with $\C(x) = \V(x)\cap C$ for all $x$. As discussed in Section \ref{ss:global_discret}, defining $C$ to be a grid with appropriate side-length is a simple way to obtain such a cover.
    
\begin{algorithm}
\caption{Semi-supervised learner (realizable)}\label{alg:semi-supervised-realizable}
\begin{algorithmic}[1]
\STATE {\bf Input:} Labeled data $S_l =  ((x_1, y_1), \ldots, (x_m, y_m))$, unlabeled data $S_u = (x_1,\ldots, x_n)$, access to RERM oracle $\A^\V_\H $, $r\gamma$-cover $C\subseteq X$
\STATE Set $\H' = \{\}$
\FOR{each labeling $y = (y_1,\ldots , y_n) \in \{0,1\}^n$ of $S_u$}\label{line:main-for-realizable}
\STATE Create labeled set $S_u^y = ((x_1, y_1), \ldots , (x_n, y_n))$.
\STATE Let $h' = \A^\V_\H(S_u^y)$
\IF{$\rLo{\V}{S^y_u}(h')=0$} 
\STATE Add $h'$ to $\H'$ 
\ENDIF
\ENDFOR
\STATE Define RERM oracle $\A^\C_{\H'}$ for induced discrete perturbation type $\C(x) = C\cap\V(x)$ 
\STATE Let $\hat{h}=\A^\C_{\H'}(S_l)$.
\STATE {\bf Output:} $\nn_C(\hat{h})$ defined by:
\STATE \qquad\qquad $\nn_C(\hat{h})(x) = \hat{h}(x')$ where $x'$ is the nearest neighbour of $x$ in $C$.
\end{algorithmic}
\end{algorithm}
\begin{theorem}
    Algorithm~\ref{alg:semi-supervised-realizable} is a $\gamma$-tolerant adversarially robust learner in the realizable setting with unlabeled sample complexity $m_u = \tilde{O}\left(\frac{\vc(\H)(d\log(1+1/\gamma)+\log d) +\log 1/\delta}{\epsilon}\right)$ and labeled sample complexity $m_l = \tilde{O}\left(\frac{\vc_\V(\H)\log m_u+\log 1/\delta}{\epsilon^2}\right)$.
\end{theorem}

\begin{proof}
    The high level idea of the proof is as follows. Since this is the realizable setting, let $h^*\in\H$ be a hypothesis such that $\rLo{\V}{P}(h^*)=0$. In the main for loop (Line~\ref{line:main-for-realizable}) of the algorithm, when the label $y = (h^*(x_1),\ldots,h^*(x_n))$ is picked, we show that the output $h'$ will satisfy $\rLo{\C}{P}(h')\leq\frac{\epsilon}{2}$ (with high probability over the randomness of $S_u$) as long as $S_u$ is of an appropriate size (to be bound later). Thus, the approximation error of the finite class $\H'$ with respect to $\rlo{\C}$ is bounded by $\frac \epsilon 2$. Next, recall that for finite perturbation types such as $\C$, the loss class $\H_\C$ of a class $\H$ of finite VC-dimension, has finite VC-dimension as well (see Lemma \ref{prop:vc_finite_loss_class}). Thus RERM learner $\A^{\C}_{\H'}$ is an agnostic learner with respect to $\rlo{\C}$ for $\H'$.
    This implies that $\rLo{\C}{P}(\hat{h})\leq\min_{h\in\H'}\rLo{\C}{P}(h)+\frac{\epsilon}{2}$ (with high probability over the randomness of $S_l$) as long as $S_l$ is of an appropriate size. Thus overall, $\rLo{\C}{P}(\hat{h})\leq \epsilon$ and this implies $\rLo{\U}{P}(\nn_C(\hat{h}))\leq \epsilon$ (as in proof of Theorem \ref{thm:supervised-global}, see Equation \ref{eqn:loss_relations} therein).
    
    Now it only remains to argue that the stated bounds on the size of $S_l$ and $S_u$ suffice for the above. Since we use $S_u$ to obtain an $h'$ that is $\frac{\epsilon}{2}$-close to $h^*$ in terms of $\rLo{\C}{P}$, an unlabeled data set size $|S_u|\geq O\left(\frac{\vc(\H)\log k + \log{1/\delta}}{\epsilon}\right)$ 
    where $k = \max_x |\C(x)|$ suffices (see Lemma \ref{prop:vc_finite_loss_class}). Using the bound on $k$, we get $O\left(\frac{\vc(\H)d(\log(1+1/\gamma)+\log d)+\log{1/\delta}}{\epsilon}\right)$. Next, to bound the size of $S_l$, note that $\H'$ is a finite set and thus $\vc(\H') \leq O(\log{|\H'|})$. Moreover, $\H'$ contains exactly one hypothesis per robustly realizable (with respect to $\V)$ labeling of $S_u$. From a simple application of Sauer lemma, we get that $|\H'|\leq O\left(|S_u|^{\vc_\V(\H)}\right)$ and thus $\vc(\H')\leq O(\vc_\V(\H)\log|S_u|)$. Thus number of labeled samples required is $|S_l| \geq O\left(\frac{\vc_\V(\H)\log|S_u|d(\log(1+1/\gamma)+\log d)+\log{1/\delta}}{\epsilon}\right)$.
\end{proof}

\subsubsection{Agnostic}
\label{sec:agnostic-semi}
For the agnostic case, the set $\H'$ formed by calling RERM for all $2^{|S_u|}$ labelings of $S_u$ can be quite large: unlike the realizable case, since the outputs of RERM no longer robustly label all points of $S_u$, the size of the resulting $\H'$ cannot be bounded by anything better than $O\left(|S_u|^{\vc(\H)}\right)$. However, we show that we can prune $\H'$ to get a new smaller set $\H''$ such that learning with respect to $\H''$ on the labeled set gives a factor-3 agnostic learner. The purpose of pruning is to ensure that if two hypotheses $h_1,h_2\in\H'$ have the property that whenever both robustly label a point in $S_u$, they give it the same label, then we keep only one of them. We still need to decide which one to keep. Our algorithm keeps the one that is robust on a larger number of points.

Similar to the realizable case, we assume we are given an $r\gamma$-cover $C$ and the corresponding perturbation type $\C$. In the agnostic case, we need to call $\text{RERM}^\C$ instead of $\text{RERM}^\V$ for creating $\H'$. As a result, the bound on $|\H''|$ is $O\left(|S_u|^{\vc_\C(\H)}\right)$ and thus the labeled sample complexity depends on $\vc_\C(\H)$, which can be bigger than $\vc_\U(\H)$, but still smaller than $\vc(\H)$. We provide a proof sketch for the guarantees of our algorithm. The detailed proof can be found in Appendix~\ref{appsec:agnostic-semi}.

\begin{algorithm}
\caption{Semi-supervised learner (agnostic)}\label{alg:semi-supervised-agnostic}
\begin{algorithmic}[1]
\STATE {\bf Input:} Labeled data $S_l =  ((x_1, y_1), \ldots, (x_m, y_m))$, unlabeled data $S_u = (x_1,\ldots, x_n)$, $r\gamma$-cover $C\subseteq X$, access to RERM oracle $\A_\H^\C$.
\STATE Set $\H' = \{\}$
\FOR{each labeling $y = (y_1,\ldots , y_n) \in \{0,1\}^n$ of $S_u$}\label{line:first-for-loop}
\STATE Create labeled set $S_u^y = ((x_1, y_1), \ldots , (x_n, y_n))$.
\STATE Add $h' = \A^\C_\H(S_u^y)$ to $\H'$.
\ENDFOR
\STATE $\H''=\{\}$
\FOR{s = 0 to $|S_u|$}
\FOR{$h\in\H'$ such that $\rLo{\C,\mathrm{mar}}{S_u}(h) = s/|S_u|$}
\IF{$\forall h'\in\H', \exists x\in S_u$ such that $\ell^{\C,\mathrm{mar}}(h,x)=\ell^{\C,\mathrm{mar}}(h',x)=0$, but $h'(x)\neq h(x)$}
\STATE Add $h$ to $\H''$
\ENDIF
\ENDFOR
\ENDFOR
\STATE Define RERM oracle $\A^\C_{\H''}$ for induced discrete perturbation type $\C(x) = C\cap\V(x)$ 
\STATE Let $\hat{h}=\text{RERM}^\C_{\H''}(S_l)$.
\STATE {\bf Output:}$\nn_C(\hat{h})$ defined by:
\STATE \qquad\qquad $\nn_C(\hat{h})(x) = \hat{h}(x')$ where $x'$ is the nearest neighbour of $x$ in $C$.
\end{algorithmic}
\end{algorithm}
\begin{theorem}
\label{thm:agnostic-semi}
    Algorithm~\ref{alg:semi-supervised-agnostic} is a factor-3 agnostic learner in the semi-supervised setting with tolerance parameter $\gamma$ with unlabeled sample complexity $m_u = \tilde{O}\left(\frac{\vc(\H)d(\log(1+1/\gamma)+\log d) +\log 1/\delta}{\epsilon^2}\right)$ and labeled sample complexity $m_l = \tilde{O}\left(\frac{\vc_\C(\H)\log m_u+\log 1/\delta}{\epsilon^2}\right)$.
\end{theorem}
\begin{proof}[Proof sketch]
The proof involves two steps. First we show $|\H''|\leq O\left(|S_u|^{\vc_\C(\H)\log|S_u|}\right)$, and then we show there exists $h''\in\H''$ such that $\rLo{\C}{P}(h'')\leq 3\cdot\rLo{\C}{P}(\H)+\epsilon$.

To bound the size, we use a more general version of Sauer lemma from~(\cite{alon2022theory} Theorem 12). Recall that the standard Sauer lemma shows that the number of different labelings induced by a hypothesis class $\H$ on $m$ points is bounded by $O\left(m^{\vc(\H)}\right)$. Here, two labelings are considered different if they are different on at least one of the $m$ points. The way we have defined $\H''$, any two $h_1,h_2\in\H''$ are \emph{robustly} different on at least one point, i.e., there exists $x\in S_u$ such that $h_1(x)\neq h_2(x)$ and $\ell^{\C,\mathrm{mar}}(h_1,x)=\ell^{\C,\mathrm{mar}}(h_2,x)=0$. The generalized Sauer lemma can be used to show that the number of such hypotheses depends on the size of the largest set they can \emph{robustly} shatter, in particular, that $|\H''|\leq O\left(|S_u|^{\vc_\C(\H)\log|S_u|}\right)$.

To show that $\H''$ contains a good hypothesis, a similar argument as in the realizable case shows that there exists $h'\in\H'$ such that $\rLo{\C}{P}(h')\leq\rLo{\C}{P}(\H)+\epsilon$, implying $\H'$ contains a good hypothesis. But the risk is that it might get pruned out in the pruning step. However, if it gets pruned, it is because of another hypothesis $h''$ that robustly labels more points of $S_u$ than $h'$ and is consistent with $h''$ on every $x\in S_u$ that is robustly labeled by both. For any $h$, let $S_u^h=\{x\in S_u~|~\ell^{\C,\mathrm{mar}}(h,x)=0\}$. Then $h'$ and $h''$ must robustly agree on at least $\frac{(1-|S_u^{h'}|+1-|S_u^{h''}|)}{|S_u|}\leq 2\cdot\frac{(1-|S_u^{h'}|)}{|S_u|}\approx 2\cdot\rLo{\C}{P}(\H)$ fraction of the points. Since $\rLo{\C}{P}(h')\approx\rLo{\C}{P}(\H)$, and $h', h''$ robustly agree on $2\cdot\rLo{\C}{P}(\H)$ fraction of points, we get that $\rLo{\C}{P}(h'')\approx 3\cdot\rLo{\C}{P}(\H)$. Here the last few steps rely on some carefully constructed generalization arguments that have been laid out in Appendix~\ref{appsec:agnostic-semi}.
\end{proof}

\section{Discussion}
The main message of this work is that adding tolerance vastly simplifies the task of adversarially robust learning, and is perhaps a more natural formulation of the robust learning problem in the first place. Recent developments in statistical learning theory often involved establishing novel sample complexity bounds or characterizations of learnability through rather impractical learners~\citep{montasser2022adversarially,BrukhimCDMY22}. 
While these provide important insights into learnability, we view our work also as promoting more PAC type analysis of methods that are closer to what is used in applications, as well as explorations into how the frameworks of analysis we choose may affect the feasibility of developing such guarantees for natural learners. Our work shows how the slight shift to tolerance for the adversarial robustness task enables PAC type guarantees for simpler learners. 
One interesting open question on a technical level is to eliminiate the dependence on $d$ for $\gamma=1$. \cite{montasser2021transductive} prove the existence of a \emph{transductive} learner for $\gamma=1$ whose sample complexity is independent of $d$. It will be nice to explore if our techniques can be used to obtain similar bounds in the inductive case. Our algorithms depend on $d$ even for $\gamma=1$ essentially because we need to cover the set $\V(x)$ with $\W$ thus giving us a cover of size $(1+1/\gamma)^d$. If, instead, we only had to cover $\U(x)$, we would get $1/\gamma^d$, which, for $\gamma=1$ becomes independent of $d$.

\bibliography{refs}

\clearpage
\appendix

\section{Additional discussion on the notion of tolerance}
\label{appsec:why-tolerance}
As has been argued before~\citep{ashtiani2023adversarially}, tolerance is a good way to capture the user's ambivalence over precisely which perturbation type to be robust against. If the user has a specific perturbation type $\U$ in mind, then we aim to find a hypothesis $h$ such that $\rLo{\U}{P}(h)\leq\rLo{\U}{P}(\H)+\epsilon$. But if the user is happy with any perturbation type between $\U$ and $\V$, we should aim for $\rLo{\U}{P}(h)\leq\rLo{\U}{P}(\H)+\xi+\epsilon$, where $\xi=\rLo{\V}{P}(\H)-\rLo{\U}{P}(\H)$ denotes the extra slack due to user's ambivalence. This gives us Definition~\ref{def:adv_learn_tol}. Below we describe another motivation to study tolerance.

\paragraph{Tolerance and learning with smoothed adversaries}
Standard learning theory tells us that learning with respect to iid data is characterized by VC-dimension, but learning when data is generated by a worst-case adversary is characterized by Littlestone dimension (which can be much bigger than VC-dimension)~\citep{shalev2014understanding}. Inspired by the smoothed analysis of algorithms~\cite{spielman2004smoothed}, it has been proved that if the worst-case adversary is ``smoothed" then learning becomes characterized by VC-dimension again~\citep{block2022smoothed, haghtalab2020smoothed,haghtalab2024smoothed}. To smooth an adversary, one takes the points generated by the worst-case adversary and perturbs them uniformly at random within a neighbourhood. In adversarially robust learning a worst-case adversary is allowed to generate points anywhere within $\U(x)$. One can consider a smoothed adversary in a similar way, where the point $z$ generated by the worst-case adversary is perturbed uniformly at random within $\W(z)$. Learning with respect to this adversary is then exactly the problem of tolerantly robust learning. Our results essentially show that adversarially robust learning with smoothed adversaries is easier.

\section{Basic VC theory}\label{appsec:vctheory}

The following definition abstracts the notion of PAC learning~\citep{vapnikcherv71, Valiant84}.

\begin{definition}[PAC Learner]\label{def:learn}
Let $\P$ be a set of distributions over $X\times Y$ and $\H$ a hypothesis class. We say $\A$ PAC learns $\H$ with respect to $\P$ with $m_\A: (0,1)^2\to \mathbb{N}$ samples if the following holds:
for every distribution $P\in \P$ over $X\times Y$, and every $\epsilon,\delta \in (0,1)$, if $S$ is an \iid sample of size at least $m_\A(\epsilon, \delta)$ from $P$, then with probability at least $1-\delta$ (over the randomness of $S$) we have
\[
\Lo{P}(\A(S)) \leq \Lo{P}(\cH) + \epsilon.
\]
$\A$ is called an \emph{agnostic learner} if $\P$ is the set of all distributions\footnote{Subject to mild measurability conditions, namely the loss sets being measurable for all $h\in\H$.} on $X\times Y$, and a \emph{realizable learner} if $\P=\{P:\Lo{P}(\cH) = 0\}$.
\end{definition}

The following is a classic result of PAC learnability for finite VC classes.
\begin{theorem}[\citep{vapnikcherv71,blumer1989learnability,haussler1992decision}]
Let $\H\subseteq \{0,1\}^X$ be a hypothesis class and let $\ell$ be a loss function such that $\vc(\H_\ell) < \infty$ is finite. 
 Then, any empirical risk minizing learner $\A$ PAC learns $\H$ with respect to loss $\ell$ with sample complexity 
  $O\left(\frac{\vc(\H_\ell)+\log(1/\delta)}{\epsilon^2}\right)$ 
   in the agnostic case and sample complexity 
   $\tilde{O}\left(\frac{\vc(\H_\ell)+\log(1/\delta)}{\epsilon}\right)$ in the realizable case.
\end{theorem}

For completeness, we here also provide the definition of the VC-dimension for collection of subsets. Note that any binary hypothesis class $\H\subseteq \{0,1\}^X$ is also a collection of subsets of $X\times\{0,1\}$.
\begin{definition}[VC-dimension]
    Let $Z$ be some domain set and $\C\subseteq 2^Z$ be a collection of subsets of $Z$. We say that some set $K\subseteq Z$ is \emph{shattered} by $\C$ if $|\{c\cap K ~:~ c\in \C\}| = 2^{|K|}$. The \emph{VC-dimension} of $\C$ is the supremum over the sizes of domain subsets that are shattered by $\C$.
\end{definition}

\section{Proof of Theorem \ref{thm:lower_bound}}
\label{appsec_lower_bound_proof}

As a reminder, we restate the Theorem here:\\

\noindent{\bf Theorem \ref{thm:lower_bound}}
{\sl For any $r\in\reals$, any $d\in \naturals$ and any $g>0$, there exist a hypothesis class $\H$ over $X = \reals^d$ with $\vc(\H) = 1$ that is not properly tolerantly robustly PAC learnable (even in the tolerantly realizable case) for $\U(x) = \B_r(x)$ and $\V(x) = \B_{(1+\gamma)r}(x)$ for any $\gamma$ with $0<\gamma \leq g$.}\\

\begin{proof}
Let $r\in\reals$, $d\in \naturals$ and $g>0$ be given. We will provide a construction with $X = \reals^1$, which can readily be embedded into any higher dimensional space. Our definition of a hypothesis class $\H$ will closely follow the construction of the hardness of proper learning in the case of standard adversarial robustness without tolerance (Theorem 1 by~\cite{MontasserHS19}). 

We construct a hypothesis class $\H$ with VC-dimension $1$ for which the VC-dimension of the adversarial loss class $\H_\V$ is arbitrarily large. Moreover, we will define $\H$ in such a way that the adversarial losses with respect to $\V$ and $\U$ are identical, and thus tolerance will not alleviate the difficulty (as it would in the original construction).

Let $n\in \naturals$ be given. We now first construct a class $\H_n$ with VC-dimension $1$ and loss class VC-dimension $n$. Consider $n$ points $x_1, x_2, \ldots, x_n \in \reals$ spaced apart with distances $|x_i - x_j| > 2r(1+g)$ for all $i\neq j$. That is, the points are positioned so that balls of radius $(1+\gamma)r$ around any two $x_i \neq x_j$ do not intersect for any $0<\gamma < g$. 
Let $(p_i)_{i\in\naturals}$ be an enumeration of prime numbers (that is $p_1 = 2, p_2 = 3, p_3 = 5$ etc). Define a one to one mapping $f$ between subsets of $[n] = \{1,2,3,\ldots, n\}$ and the first  $2^n$ prime numbers. 

For each subset $Z\subseteq [n]$, we define a hypothesis $h_Z$ as follows:
\[
h_Z(x) = \begin{cases}
1 \quad\text{if } x  = x_j + \frac{r}{(p_{f(Z)})^m} \text{ for } j\in Z, m\in\naturals\\
0 \quad\text{otherwise}
\end{cases}
\]
That is, each hypothesis $h_Z$ labels most of $X = \reals$, and in particular all points $x_1, x_2, \ldots, x_n$ with $0$. The only points $h_Z$ labels with $1$ are the above defined sequences approaching the points $x_j$ where $j$ is an index in $Z$. These sequences are the points $x_j + \frac{r}{(p_{f(Z)})^m}$  for $m\in\naturals$, that is defined by a prime number that is uniquely associated with the set $Z$. 

Now we set $\H_n$ be the set of all these hypotheses, that is
$\H_n = \{h_Z ~:~ Z\subseteq [n]\}$.
Note that in this construction, every point in $X = \reals$ gets labeled with $1$ by at most one hypothesis in $\H_n$, and thus $\vc(\H_n) = 1$.

Now note that for a subset $Z\subseteq [n]$, and an $j\in Z$, the point $x_j$ is arbitrarily close to points that are labeled with $1$ by $h_Z$, and thus the robust loss $\rlo{\B}(h_Z, x_j, 0 ) = 1$ for any ball perturbation type $\B(x) = \B_\rho(x)$ for any (arbitrarily small) radius $\rho >0$. In particular $\rlo{\U}(h_Z, x_j, 0 ) = \rlo{\V}(h_Z, x_j, 0 ) = 1$ if (and only if) $j\in Z$. This shows that the VC dimension of the robust loss class $\vc(\H_\V) = \vc(\H_\U) = n$.

From here the argument for impossibility of proper learning in the $(\U, \V)$-tolerantly realizable setting proceeds as in the proof of Theorem 1 by \cite{MontasserHS19}. We restrict the class $\H_n$ to only contain functions corresponding to subsets $n/2$:
\[
\H_n' = \{h_Z ~:~ Z\subseteq [n], |Z| = n/2\}.
\]
Now, we consider the set of distributions $\P_n$ that ditribute their mass uniformly over $n/2$ domain points among the $x_1, x_2, \ldots x_n$ with label $0$. These distributions are $\V$-robustly realizable by 
$\H_n'$, in particular the function $h_Z$ for $Z$ that is the complement of the distribution's support in $\{x_1, x_2, \ldots x_n\}$ has robust loss $0$ with respect to $\V$. However, standard arguments show that any \emph{proper} learner that sees only samples of sizes $n/4$ cannot correctly identify this required complement set and thus has high probabitiy of outputting a function $h_Z\in\H_n'$ where $Z$ intersects significantly with the support of the distribution and thus suffers high robust loss. In particular it suffers this high robust loss with respect to $\V$ and equally with respect to $\U$ or balls of \emph{any} smaller radius. Thus, the relaxation to the tolerant setting does not facilitate proper learning in this case. 

Finally, by copying the above construction into disjoint intervals of $X = \reals$ for all $m\in\naturals$ the resulting class $\H = \bigcup_{m\in\naturals} \H_m'$ has VC dimension $\vc(\H) = 1$, but cannot be learned by any proper learner in the tolerantly realizable setting.

\end{proof}

\section{Local discretization for both realizable and agnostic setting}\label{appsec:local_discretization}

A slightly modified version of our first method Algorihm \ref{alg:supervised} can be shown to also work in the agnostic case. For completeness, we state the modified algorithm and complete and unified analysis for both realizabe and agnostic settings here. Algorithm \ref{alg:supervised_local_agnostic} below is a variant where the initial RERM hypothesis is chosen with respect to a discretized type $\C\prec\V$. 

\begin{algorithm}
\caption{RERM-and-Smooth}\label{alg:supervised_local_agnostic}
\begin{algorithmic}
\STATE {\bf Input:} Data $S =  ((x_1, y_1), \ldots, (x_m, y_m))$, access to an RERM oracle $\A^\V_\H$, and smoothing perturbation type $\W$.
\STATE Set $\hat{h} = \A^{\V}_{\H}(S)$
\STATE {\bf Output:} $\sm_{\W}(\hat{h})$ defined by 
\STATE \qquad\qquad $\sm_{\W}(\hat{h})(x) = \indct{\Ex_{x'\sim\W(x)} \hat{h}(x') \geq 1/2}$
\end{algorithmic}
\end{algorithm}

\begin{theorem}
Let $\H$ be a hypothesis class of finite VC-dimension ($\vc(\H) <\infty$), let $0< \eta < 1/3$, and let $\V,\U, \W$ and $\C$ be perturbation types that satisfy \\
$\bullet$ $\U\prec \V$ and $\C\prec \V$, \\ 
$\bullet$ $\W(z)\in\V(x)$ for all $x\in X$ and $z\in \U(x)$,\\
$\bullet$ $\C(x)$ is finite and an $\eta$-net of $\H$ with respect the uniform measure $\mu_{\V(x)}$ over $\V(x)$ for all $x\in X$,\\
$\bullet$ $\mu_{\V(x)}(\W(z))\geq 3\eta$ for all $x\in X$ and $z\in\U(x)$.

Then Algorithm \ref{alg:supervised} $(\U,\V)$-tolerant robustly PAC learns $\H$. Moreover, the perturbation type $\C$ can be chosen so that the resulting sample complexity in the tolerantly realizable case is bounded by
    \[
    m(\epsilon, \delta) = 
    \tilde{O}\left(\frac{\vc(\H)\log(\vc(\H)/\eta) + \log(1/\delta)}{\epsilon} \right)
    \]
    and in the agnostic case by
   \[
    n(\epsilon, \delta) = 
    \tilde{O}\left(\frac{\vc(\H)\log(\vc(\H)/\eta) + \log(1/\delta)}{\epsilon^2} \right)
    \]
\end{theorem}

\begin{proof}
Let $h$ be some hypothesis and let $\sm_{\W}(h)$ be the $\W$-smoothed version of $h$ as defined above. We start by proving that for any distribution $P$, we have 
\begin{equation}\label{eqn:loss_relations_smooth_appendix}
\rLo{\U}{P}(\sm_{\W}(h))  ~\leq~ \rLo{\C}{P}(h) ~\leq~ \rLo{\V}{P}(h)
\end{equation}
The second inequality is immediate by observing that $\C(x) \subseteq \V(x)$ for all $x\in X$. To prove the the first inequality we will show that it holds pointwise, that is $\rlo{\U}(\sm_{\W}(h), x, y)  \leq \rlo{\C}(h,x,y)$ for all $(x, y)\in X\times Y$.

Indeed, assume that for some $(x,y) \in X \times Y$ 
we have $\rlo{\U}(\sm_{\W}(h), x, y) = 1$. This means that there exists some $z\in \U(x)$ with $\sm_{\W}(h)(z) \neq y$. This 
implies that $\indct{\Ex_{x'\sim\W(z)} \hat{h}(x') \geq 1/2} \neq y$, which means that $\Pr_{x'\sim \W(x)}[\hat{h}(x')\neq y] > 1/2$. Now 
$\mu_{\V(x)}(\W(z))\geq 3\eta$ 
together with $\W(z)\subseteq\V(x)$ 
(where $\mu_{\V(x)}$ is a uniform measure over $\V$) 
implies that $\Pr_{x'\sim \V(x)}[\hat{h}(x')\neq y] \geq (3/2)\eta> \eta$. 
Now, since $\C(x)$ is an $\eta$-net with respect to $\mu_{\V(x)}$ for $\H$, this implies that there exists a $c\in \C(x)$ with $\hat{h}(c) \neq y$. Thus, we have $ \rlo{\C}(h,x,y) = 1$ which is what we needed to show.

Since $\vc(\H)$ is finite, $\C$ can be chosen 
so that $|\C(x)| \leq 3\vc(\H)/\eta$ and satisfy $\C(x)$  being an $\eta$-net for $\H$ with respect to $\mu_{\V(x)}$ for each $\V(x)$.
Note that since the perturbation sets of type $\C$ are finite and their sizes are uniformly upper bounded by $3\vc(\H)/\eta$, the VC-dimension of the loss class of $\H$ with respect to $\rlo{\C}$ is bounded by $\vc(\H)\log(3\vc(\H)/\eta)$ (see Lemma \ref{prop:vc_finite_loss_class}).
\paragraph{Realizable case.}
Note that, since $\C \prec \V$, any distribution $P$ that is realizable by $\H$ with respect to $\rlo{\V}$ (as assumed in the tolerantly realizable case) is also realizable with respect to $\rlo{\C}$. Now the above bound on the VC-dimension on the loss class with respect to $\rlo{\C}$ implies that any RERM learner is a successful robust PAC learner for perturbation type $\C$, and thus, with the stated samples sizes, with high probability at least $1-\delta$ we have $\rLo{\C}{P} (\hat{h}) \leq \epsilon$. Now Equation \ref{eqn:loss_relations_smooth} implies $\rLo{\U}{P}(\sm_{\W}(\hat{h})) \leq \epsilon$ as required.
\paragraph{Agnostic case.}
Note that Equation \ref{eqn:loss_relations_smooth} also implies that for any distribution $P$, the approximation error of $\H$ with respect to $\rlo{\C}$ is upper bounded by the approximation error of $\H$ with respect to $\rlo{\V}$
\[
\rLo{\C}{P}(\H) ~\leq ~ \rLo{\V}{P}(\H)
\]
By the bound on the VC-dimension of the loss class of $\H$ with respect to $\rlo{\C}$, any empirical risk minimizing learner for $\H$ with respect to $\rlo{\C}$ outputting $\hat{h}$, is a successful robust PAC learner with respect to $\rlo{C}$, yielding
\[
\rLo{\C}{P}(\hat{h}) \leq \rLo{\C}{P}(\H) +\epsilon ~\leq ~ \rLo{\V}{P}(\H) +\epsilon
\]
with high probability at least $1-\delta$ over the training samples for the stated sample sizes in the agnostic case.
Combining this expression with  Equation \ref{eqn:loss_relations_smooth}, we obtain 
\[
\rLo{\U}{P}(\sm_{\W}(\hat{h})) ~\leq ~ \rLo{\C}{P}(\hat{h}) ~\leq~ \rLo{\V}{P}(\H) +\epsilon
\]
as required for tolerantly robust learning. This completes the proof of the theorem.

\end{proof}

\begin{remark}
  For the case that $\V$, $\U$ and $\W$ are Euclidian balls of radii $r(1+\gamma)$, $r$ and $r\gamma$ respectively in $\reals^d$, have $\mu_{\V(x)}(\W(z)) = \frac{\gamma^d}{(1+\gamma)^d}$. Thus we can chose $\eta = \frac{1}{3(1+1/\gamma)^d}$ and obtain sample complexity bounds
    \[
    m(\epsilon, \delta) = 
    \tilde{O}\left(\frac{\vc(\H)(\log(\vc(\H)) + d \log(1+1/\gamma)) + \log(1/\delta)}{\epsilon} \right)
    \]
in the tolerantly realizable case and
   \[
    n(\epsilon, \delta) = 
    \tilde{O}\left(\frac{\vc(\H)(\log(\vc(\H)) + d \log(1+1/\gamma)) + \log(1/\delta)}{\epsilon^2} \right)
    \]
    in the agnostic case.
\end{remark}

\section{Proofs for agnostic semi-supervised learning}
\label{appsec:agnostic-semi}
\noindent{\bf Theorem~\ref{thm:agnostic-semi}}
{\sl Algorithm~\ref{alg:semi-supervised-agnostic} is a factor-3 agnostic learner in the semi-supervised setting with tolerance parameter $\gamma$ with unlabeled sample complexity $m_u = \tilde{O}\left(\frac{\vc(\H)d(\log(1+1/\gamma)+\log d) +\log 1/\delta}{\epsilon^2}\right)$ and labeled sample complexity $m_l = \tilde{O}\left(\frac{\vc_\C(\H)\log m_u+\log 1/\delta}{\epsilon^2}\right)$.}\\

\begin{proof}
    We divide the proof into two parts. First, we show that $\H''$ has a small size, and then we show that $\H''$ contains a hypothesis whose robust loss with respect to $\C$ is not much worse compared to the optimal hypothesis $h^*$. These two facts combined show that $\hat{h}$ gets a small robust loss with respect to $\C$. This means $\nn_C(\hat{h})$ gets a small robust loss with respect to $\V$. 

    To show that $\H''$ has a small size we use a result from~\citep{alon2022theory}. For any hypothesis $h\in\H$, define a ``partial" hypothesis $\partialh{h}$ as follows. We say $\partialh{h}(x) = 0$ if $h(z) = 0$ for every $z\in\C(x)$, $\partialh{h}(x) = 1$ if $h(z) = 1$ for every $z\in\C(x)$, and $\partialh{h}(x)=\star$ otherwise. Transforming every hypothesis of $\H$ in this way gives us a new hypothesis class $\partialh{\H}$. We call such a hypothesis class a partial hypothesis class, and a class that does not contain any partial hypotheses a total hypothesis class. Given any set $S=\{x_1,\ldots,x_m\}\in X^m$, we say that $\partialh{\H}$ shatters $S$ if for every labeling $\{y_1,\ldots,y_m\}\in\{0,1\}^m$, there exists some $\partialh{h}\in\partialh{\H}$ such that $\partialh{h}(x_i) = y_i$ for every $x_i \in S$. Note that even though $\partialh{h}$ is allowed to output $\star$, it is supposed to shatter a set only using the labels 0 and 1. The VC-dimension of the class $\partialh{\H}$ is defined as the size of the biggest set that is shattered by it. It is easy to see that this is equal to $\vc_\C$. We restate the following lemma for our context, which was originally shown in~\citep{alon2022theory}.
    
    \begin{lemma}[~\cite{alon2022theory} Theorem 12] 
    \label{lemma:partial}
    Given the set $S_u$ and the class $\partialh{\H}$ of partial hypotheses, there exists a class $\tilde{\H}$ of total hypotheses such that $|\tilde{\H}|\leq |S_u|^{O\left(\vc_\C(\H)\log |S_u|\right)}$, and for every $\partialh{h}\in\partialh{\H}$, there exists $\tilde{h}\in\tilde{\H}$ such that $\tilde{h}(x) = \partialh{h}(x)$ for every $x\in S_u$ satisfying $\partialh{h}(x)\neq\star$.   \end{lemma}

    To see that $\H''$ is small, consider the corresponding partial version $\partialh{\H''}$. Using Lemma~\ref{lemma:partial}, we can construct a total hypothesis class $\tilde{\H}''$ that has a hypothesis $\tilde{h}$ for every $h\in\H''$ such that $h(x) = \tilde{h}(x)$ for every $x\in S_u$ that satisfies $\ell^{\C,\mathrm{mar}}(h, x)=0$ (i.e., $h$ is robust on $x$). Moreover, we can see that for each $h$, the corresponding $\tilde{h}$ is unique. This is because due to the way $\H''$ is constructed, for any two $h_1,h_2\in\H''$, there exists $x\in S_u$ such that $h_1(x)\neq h_2(x)$ and $\ell^{\C,\mathrm{mar}}(h_1,x)=\ell^{\C,\mathrm{mar}}(h_2,x)=0$. This means the corresponding $\tilde{h_1}$ and $\tilde{h_2}$ must be different and $x$ is the witness. Since there is a unique hypothesis in $\tilde{\H}''$ for every hypothesis in $\H''$, we get $|\H''|\leq |\tilde{\H}''|\leq |S_u|^{O\left(\vc_\C(\H)\log |S_u|\right)}$.

    Next, we show that $\H''$ contains a ``good" hypothesis. Imagine that $S_u$ was generated by first sampling a labeled set of size $|S_u|$ from distribution $P$ and then removing the labels. Let $y = (y_1,\ldots, y_m)$ be the labels that were removed. Then, at some point $y$ will be considered by the first for loop (Line~\ref{line:first-for-loop}) in Algorithm~\ref{alg:semi-supervised-agnostic} and a corresponding hypothesis $h'$ will be added to $\H'$. Since $h'$ is the result of running $\text{RERM}^\C$, using Lemma~\ref{prop:vc_finite_loss_class}, we know that with high probability over the randomness of $S_u$, we have $\rLo{\C}{P}(h')\leq\rLo{\C}{P}(\H)+\epsilon$. Thus $\H'$ does contain a good hypothesis. But the risk is that it might get pruned out in the pruning step. 

    We can show that because of the way we do the pruning, there will be another hypothesis $h''\in\H''$ that is not much worse than $h'$ on $S_u$. But we want a guarantee wrt $P$, and thus we need a uniform convergence result that links the closeness of two hypotheses on $S_u$ with their closeness on $P$.

    We use the following lemmas.
    \begin{definition}
        For any two hypotheses $h_1,h_2\in\H$ and $x\in X$, define:
        \[
        \ell^{\C,\text{par}}(h_1,h_2,x)=\indct{\partialh{h_1}(x)\neq\partialh{h_2}(x)}.
        \]
        Also, define $\rLo{\C,\text{par}}{S_u}(h_1,h_2)$ and $\rLo{\C,\text{par}}{P}(h_1,h_2)$ appropriately. 
    \end{definition}
    \begin{lemma}
        \label{lemma:uc-pair}
    For $S_u \geq O\left(\frac{\vc(\H)d(\log(1+1/\gamma)+\log d)+\log 1/\delta}{\epsilon^2}\right)$, with probability at least $1-\delta$ over the randomness of $S_u$, we have that $|\rLo{\C,\text{par}}{S_u}(h_1,h_2)-\rLo{\C,\text{par}}{P}(h_1,h_2)|\leq \epsilon$
    \end{lemma}
    \begin{proof}
Let $\max_{x\in X} |\C(x)|\leq k$ and consider the hypothesis class $\H\times\H$ consisting of pairs of hypotheses from $\H$. For any $h_1,h_2\in\H$ define the function $(h_1,h_2): X\rightarrow Y$ as $(h_1,h_2)(x) = \ell^{\C,\text{par}}(h_1,h_2,x)$. We show that $\vc(\H\times\H)\leq O(\vc(\H)\log k)$. To see this, consider a sequence $S=(x_1,\ldots,x_m)$ of $m$ points from $X$ and count the number of patterns induced on it by members of $\H\times\H$. It is easy to see that for any $x\in X$, the value of $(h_1,h_2)(x)$ is uniquely determined once we specify the values of $h_1(x), h_2(x), \ell^{\C,\mathrm{mar}}(h_1,x)$, and $\ell^{\C,\mathrm{mar}}(h_2,x)$. Thus the total number of patterns is at most the product of the number of patterns induced by each and thus can be bounded by $m^{O(\vc(\H)\log k)})$. Standard uniform convergence results for VC classes concludes the proof.
\end{proof}

\begin{lemma}
\label{lemma:par-vs-robust}
    For any two hypotheses $h_1, h_2$, $|\rLo{\C}{P}(h_1)-\rLo{\C}{P}(h_2)|\leq\rLo{\C,\text{par}}{P}(h_1, h_2)$.
\end{lemma}
\begin{proof}
    In fact, it is easy to see that for any $x$, if $\ell^{\C,\text{par}}(h_1,h_2,x)=0$, then for all $y\in\{0,1\}$, $\ell^\C(h_1,x,y)=\ell^\C(h_2,x,y)$. Thus the lemma follows. 
\end{proof}

Now, we are ready to complete the proof. If $h'\in\H''$, then $\H''$ contains a good hypothesis. Otherwise, if $h'$ gets pruned out, that can only be because there was another hypothesis $h''$ such that $|S^{h''}_u|\geq |S^{h'}_u|$ and $h'(x) = h''(x)$ for all $x\in S^{h'}_u\cap S^{h''}_u$. We have:
\begin{align}
    \rLo{\C,\text{par}}{S_u}(h',h'') & \leq \frac{(1-|S_u^{h'}|) + (1-|S_u^{h''}|)}{|S_u|} \notag \\
    & \leq 2\cdot\frac{1-|S_u^{h'}|}{|S_u|} \notag\\
    & = 2\cdot\rLo{\C,\mathrm{mar}}{S_u}(h') \notag\\
    & \leq 2\cdot\rLo{\C,\mathrm{mar}}{P}(h') + 2\epsilon \notag\\
    & \leq 2\cdot\rLo{\C}{P}(h') + 2\epsilon \notag\\
    & = 2\cdot\rLo{\C}{P}(\H) + 2\epsilon\label{eqn:ref}
\end{align}

Finally, when we run $\text{RERM}^\C$ on $S_l$ with respect to $\H''$, we get $\hat{h}$, that satisfies:
\begin{align}
    \rLo{\C}{P}(\hat{h}) & \leq \rLo{\C}{P}(\H'') + \epsilon\notag \\ 
    & \leq \rLo{\C}{P}(h'') + \epsilon \notag\\
    & \leq \rLo{\C}{P}(h') + \rLo{\C,\text{par}}{P}(h'',h') + 2\epsilon \label{eqn:lemma}\\
    & \leq \rLo{\C}{P}(\H) + \rLo{\C,\text{par}}{S_u}(h'',h') + 3\epsilon \label{eqn:uc}\\
    & \leq 3\cdot\rLo{\C}{P}(\H) + 5\epsilon\label{eqn:four}
\end{align}
Here,~\eqref{eqn:lemma} follows from Lemma~\ref{lemma:par-vs-robust}, ~\eqref{eqn:uc} follows from Lemma~\ref{lemma:uc-pair} and~\eqref{eqn:four} follows from ~\eqref{eqn:ref}. Thus overall, $\rLo{\C}{P}(\hat{h})\leq \epsilon$ and this implies $\rLo{\U}{P}(\nn_C(\hat{h}))\leq \epsilon$ (as in proof of Theorem \ref{thm:supervised-global}, Equation \ref{eqn:loss_relations}).
\end{proof}
\end{document}